\documentclass[conference]{IEEEtran}
\usepackage{times}

\usepackage[numbers]{natbib}
\usepackage{multicol}
\usepackage[bookmarks=true]{hyperref}

\usepackage{amsmath}
\usepackage{amssymb}
\usepackage{amsthm}
\usepackage{bm}
\usepackage{graphicx}
\usepackage{float}
\let\labelindent\relax
\usepackage[mathscr]{euscript}
\usepackage{float}
\usepackage[version=4]{mhchem}
\usepackage{siunitx}
\usepackage{longtable,tabularx}
\usepackage{caption}
\usepackage{subcaption}
\usepackage[utf8]{inputenc}
\usepackage{tikz}
\let\labelindent\relax

\usepackage{booktabs}
\usepackage{multirow}

\usepackage{setspace}
\usepackage[normalem]{ulem}
\usepackage{array}

\usepackage{thmtools}
\usepackage{thm-restate}

\usepackage{algorithm}
\usepackage[noend]{algpseudocode}
\makeatletter
\def\BState{\State\hskip-\ALG@thistlm}
\makeatother
\algnewcommand{\Or}{\textbf{or}\,}
\algnewcommand\algorithmicswitch{\textbf{switch}}
\algnewcommand\algorithmiccase{\textbf{case}}
\algdef{SE}[SWITCH]{Switch}{EndSwitch}[1]{\algorithmicswitch\ #1\ \algorithmicdo}{\algorithmicend\ \algorithmicswitch}%
\algdef{SE}[CASE]{Case}{EndCase}[1]{\algorithmiccase\ #1}{\algorithmicend\ \algorithmiccase}%
\algdef{SE}[ELSE]{Else}{EndElse}[1]{\algorithmicelse\ #1}{\algorithmicend\ \algorithmicelse}%
\algtext*{EndSwitch}%
\algtext*{EndCase}%
\algtext*{EndElse}%

\def\myMSFigureScale{0.22}

\def\myLineScale{1}

\newcommand{\tikzcircle}[2][red,fill=red]{\tikz[baseline=-0.5ex]\draw[#1,radius=#2] (0,0) circle ;}
\newcommand{\tikzsquare}[2][red,fill=red]{\tikz\draw[#1] (0,0) rectangle (#2, #2) ;}
\newcommand{\tikzline}[2][red,thick]{\tikz\draw[#1,#2] (0,0) -- (0.5em, 0.5em) ;}

\newcommand{\R}{\mathbb{R}}

\newcommand{\set}[1]{\{#1\}}

\newcommand{\argmin}{\mathop{\mathrm{argmin}}}

\renewcommand{\baselinestretch}{0.98}
\makeatletter
\newcommand*{\defeq}{\mathrel{\rlap{\raisebox{0.3ex}{$\m@th\cdot$}}\raisebox{-0.3ex}{$\m@th\cdot$}}=}
\makeatother

\begin{document}

\title{Lazy Lifelong Planning for Efficient Replanning \\ in Graphs with Expensive Edge Evaluation}


\author{\authorblockN{Jaein Lim}
\authorblockA{School of Aerospace Engineering\\
Georgia Institute of Technology\\
Atlanta, Georgia 30332--0150\\
Email: jaeinlim126@gatech.edu}
\and
\authorblockN{Siddhartha Srinivasa}
\authorblockA{ School of Computer Science\\ \& Engineering\\
University of Washington \\
Seattle, WA, 98195-2355\\
Email: siddh@cs.uw.edu}
\and
\authorblockN{Panagiotis Tsiotras}
\authorblockA{School of Aerospace Engineering\\
Institute for Robotics \& Intelligent Machines\\
Georgia Institute of Technology\\
Atlanta, Georgia 30332--0150\\
Email: tsiotras@gatech.edu}}


%

\maketitle

\begin{abstract}
We present an incremental search algorithm, called Lifelong-GLS, which combines the vertex efficiency of Lifelong Planning A* (LPA*) and the edge efficiency of Generalized Lazy Search (GLS) for efficient replanning on dynamic graphs where edge evaluation is expensive. We use a lazily evaluated LPA* to repair the cost-to-come inconsistencies of the relevant region of the current search tree based on the previous search results, and then we restrict the expensive edge evaluations only to the current shortest subpath as in the GLS framework. 
The proposed algorithm is complete and correct in finding the optimal solution in the current graph, if one exists. 
We also show that the search returns a bounded suboptimal solution, if an inflated heuristic edge weight is used and the tree repairing propagation is truncated early for faster search.
Finally, we show the efficiency of the proposed algorithm compared to the standard LPA* and the GLS algorithms over consecutive search episodes in a dynamic environment. 
For each search, the proposed algorithm reduces the edge evaluations by a significant amount compared to the LPA*. Both the number of vertex expansions and the number of edge evaluations are reduced substantially compared to GLS, as the proposed algorithm utilizes previous search results to facilitate the new search. 
\end{abstract}

\IEEEpeerreviewmaketitle

\section{Introduction}
Plans change in the real world. This is because obtaining accurate models of the complex world is difficult, and the models themselves become quickly out of date when the world is uncertain or changing. 
Hence, replanning is an essential problem for every decision-making agent with partial knowledge operating in a dynamic environment. 
The need for efficient replanning has been manifested in a wide range of applications, typically in situations where the world is abstracted via graph representations. 
Such abstractions allow tractable search algorithms to find an optimal path in the given graph, but when the underlying graph changes because either the world or the model of the world changes, then the plan needs to be updated accordingly. 

Consider the following examples. A mobile robot traversing through an unknown terrain needs to repair the path whenever the path is found to be infeasible, or a better path becomes available as the robot gains more information about the terrain~\cite{Koenig2005}. 
For sampling-based motion-planning problems, where the search space is asymptotically approximated with a series of graphs of increasing density, the choice of the replanning strategy to improve the current search tree dictates the convergence rate to an asymptotically optimal solution ~\cite{Arslan2013, Gammell2015, Strub2020a, Strub2020b}. 
For distributed multi-agent problems with cooperative communication~\cite{Torreno2014, Nissim2014, Lim2020}, where a planning entity possesses only local perception of the search space, each agent must therefore resolve any inconsistencies revealed online as the communication refines the local perception. Replanning is necessary to refine each local plan to achieve global consensus. 

Incremental search methods~\cite{Ramalingam1996, Koenig2004} store the previous search tree in order to identify the inconsistent portion of the tree when the graph changes in order to efficiently repair the current tree. 
Any identified inconsistencies are propagated onward to make the search tree consistent again with respect to the current graph changes without having to solve the problem from scratch.
In particular, Lifelong Planning A* (LPA*)~\cite{Koenig2004} efficiently restricts repairs to only the optimal path candidate guided by a consistent heuristic and a priority queue similar to that of the A* algorithm~\cite{Hart1968}. This means that LPA* heuristically delays the expansion of inconsistent vertices until repairing becomes necessary in order to find the new optimal solution with respect to the current graph.
Given a modified graph, LPA* is provably efficient in the sense that a vertex is never expanded more than twice and any inconsistent vertices outside the relevant region are never expanded~\cite{Koenig2004}. Hence, LPA* can find the new optimal solution significantly faster than searching from scratch, especially when the change is small and less relevant to the new optimal path. 
This efficiency of LPA* has been the backbone to numerous applications in which re-planning is crucial~\cite{Koenig2005, Arslan2013, Gammell2015, Strub2020a, Strub2020b}.  

Unfortunately, the design of LPA* is tailored to reducing the number of vertex expansions to find the new optimal solution, and it is indifferent to the number of edge evaluations. 
This property of LPA* often results in unnecessarily excessive edge evaluations to find the new optimal solution, causing significant overhead in problem domains where edge evaluation dominates computation time. 
For example, in motion planning problems~\cite{Kavraki1996, Lavalle1998, LaValle2001, Webb2013}, an edge evaluation consists of multiple collision checks in the configuration space, solving two-point boundary value problems, or propagating the system dynamics with a closed-loop controller. 
In this paper, we seek to remedy the excessive edge evaluations of LPA* by borrowing ideas from the lazy search framework of~\cite{Bohlin2000, Cohen2014, Hauser2015, Dellin2016, Mandalika2018, Mandalika2019}. Before we delve into these ideas, let us first characterize two aspects of LPA* that attribute to excessive edge evaluations.

LPA* needs to update the edge values of all changed edges compared to the previous graph, in order to identify the inconsistent vertices so that repairing propagation can begin in the current graph. 
LPA* evaluates all these changed edges before repairing propagation commences, regardless of whether these changes are relevant to the current problem or not. 
Note that among the inconsistent vertices identified by these edge evaluations, only the relevant inconsistent vertices are eventually expanded by LPA* during repair propagation. 
In other words, if an edge evaluation results in some inconsistency that is irrelevant to the current problem such that LPA* never uses this new information to find the new optimal solution, then this edge evaluation is not necessary.   
When the graph changes significantly, evaluating those edges can be very expensive. 

LPA* also repairs the inconsistent part of the search tree in the same way A* expands the frontier vertices with the lowest cost estimates (the so called, $f$-value) as a best-first search. This implies that LPA* requires exhaustive edge evaluations upon expanding a vertex, as A* would evaluate all the incident edges upon expanding a vertex. 
Specifically, LPA* needs to evaluate all the incident edges when an inconsistent vertex is expanded to find a new optimal parent vertex, which then leads to the propagation of the inconsistency information to all of its children vertices, and so on.
This behavior is often referred as a ``zero-step lookahead" in the literature~\cite{Mandalika2018, Mandalika2019}, where no heuristic estimate of the edge value is utilized to prioritize the next best vertex to expand.\footnote{This terminology is different from the one-step lookahead used in LPA*~\cite{Koenig2004}, which refers to the $rhs$-value of an inconsistent vertex being one-step better informed than its $g$-value.} Hence, expanding a vertex incurs actual evaluations of all incident edges, regardless of each edge’s potential to be a part of the shortest path.   

The issue of excessive edge evaluations has been explicitly addressed within the lazy search framework in order to reduce the actual number of edge evaluations by delaying these evaluations as much as possible~\cite{Bohlin2000, Cohen2014, Hauser2015, Gammell2015, Dellin2016, Mandalika2018, Mandalika2019}. 
The main idea of the lazy search framework is to delay the actual evaluation of the edges using a $n$-step lookahead ($n>0$), by prioritizing the expansion of the subpath constrained with an $n$-number of heuristically evaluated edges.
For example, when a lazy search algorithm with the one-step lookahead expands a vertex, it heuristically estimates the values of the incident edges instead of actually evaluating them, unless the values are already known. Then, the children vertices are inserted in the priority queue with the total cost estimate of the path constrained to this heuristically evaluated edge. 
The edge is actually evaluated only when the child vertex (equivalently, the subpath to the child vertex) is chosen to be expanded. Various algorithms such as Lazy Weighted A* (LWA*)~\cite{Cohen2014}, Batch Informed Trees* (BIT*)~\cite{Gammell2015}, and Class Ordered A* (COA*)~\cite{Lim2020b} use this one-step lookahead strategy to mitigate excessive edge evaluations. 

In \cite{Mandalika2018} it was shown that the number of edge evaluations decreases as the lookahead steps increase.
In fact, using an infinite lookahead step (LazyPRM~\cite{Bohlin2000}, LazySP~\cite{ Dellin2016}), i.e., restricting the edge evaluations to the shortest path to the goal (instead of subpaths), is proven to be edge optimal,  that is, the number of edge evaluations is minimized.
In essence, infinite-lookahead algorithms heuristically grow a search tree until the path to the goal is found, and only then the edges along the heuristically shortest path are evaluated to disprove its feasibility. 
The underlying heuristic search tree is repaired with the true edge values upon the actual evaluation, so that any infeasible shortest path is eliminated. 
This is iterated until all edges are feasible along the current shortest path,  otherwise no solution exists. 

It is worth noting that the edge optimality of LazySP comes at the expense of many vertex expansions. This is because the heuristic tree is grown beyond a possibly infeasible edge, and therefore, the subtree must be repaired when the edge is revealed to be infeasible upon evaluation.
On the other hand, zero-step lookahead algorithms, such as the A* algorithm, do not grow the subtree beyond any infeasible edge, therefore minimizing the number of vertex expansions.
In \cite{Mandalika2018} the relationship between the number of lookahead steps and the total computation time to solve the problem has been studied extensively to highlight the tradeoffs between vertex rewiring and edge evaluation in different problem domains.

An $n$-step lookahead algorithm (e.g., LRA*~\cite{Mandalika2018}) strikes a balance between edge evaluation and vertex expansions by growing the heuristic tree with a number of unevaluated edges before the evaluation reveals edge feasibility along the subpath to the goal on the heuristic tree. 
Finally, Generalized Lazy Search (GLS) encompasses various lookahead strategies with a user-defined algorithmic toggle between vertex rewiring and edge evaluation \cite{Mandalika2019}. With a proper choice of the toggle from the search and the evaluation, GLS hence reduces to LazySP, LRA*, or LWA*. 

In this paper, we extend GLS to incorporate lifelong planning behavior, by maintaining a lazy LPA* search tree with non-overestimating heuristic edges. In other words, we restrict the actual edge evaluations of LPA* to only those edges that could possibly be part of the optimal path in the current graph.
We reduce the excessive edge evaluations of LPA* in terms of the two aspects discussed above: any irrelevant edges will never be evaluated upon  graph changes before the search, and only the edges that could possibly be on the optimal path of the current graph will be evaluated during the search.
We leave the choice of the lookahead strategy to be quite general, as in the GLS framework, to allow for a tradeoff between the search and the evaluation steps to adapt to different problem domains. 
Hence, we attain a version of Lifelong-LazySP on one end, and a version of Lifelong-LWA* on the other end, by adopting an infinite or an one-step lookahead, respectively. 

The proposed algorithm, Lifelong-GLS (L-GLS) is complete and finds the optimal solution in the current graph. 
Compared to GLS, the proposed algorithm can possibly find the optimal solution faster by reusing previous search results. 
Compared to LPA*, our algorithm reduces significantly the number of edge evaluations. 
Moreover, when the heuristic edge values are inflated and the inconsistency repairing step is truncated for faster search, then the solution returned by L-GLS is bounded suboptimal.

\section{Problem Formulation}

We first introduce the variables and relevant notation that will be used throughout the rest of the paper. 
\subsection{Lazy Weight Function}
Let $G =(V,E)$ be a graph with vertex set $V$ and edge set $E.$
For a vertex $v\in V$, we denote the predecessor vertices of $v$ with $pred(v)$ and its successor vertices with $succ(v)$.
For each edge $e\in E$, a weight function $w: E\to (0,\infty]$ assigns a positive real number, including infinity, to this edge, e.g., the distance to traverse this edge, and infinity if traversing the edge is infeasible. 
Also, we denote an admissible heuristic weight function with $\widehat{w}: E\to (0,\infty),$ which assigns to an edge a non-overestimating positive real number such that $\widehat{w}(e)\leq w(e)$ for all $e\in E.$ 
We assume that evaluating the true weight $w$ is computationally expensive, but the heuristic edge $\widehat{w}$-value is easy to compute.
Let $E_\mathrm{eval} \subseteq E$ be the set of all evaluated edges, that is, all edges whose $w$-values have been computed.
We introduce a lazy weight function $\overline{w} : E \to (0,\infty]$ which assigns to an edge its admissible heuristic weight $\overline{w}$ before the evaluation and its true weight $w$ after the evaluation, i.e.,
\begin{equation}
	\overline{w}(e) \defeq \begin{cases}
		w(e), & \mathrm{if}\; e\in E_\mathrm{eval}, \\
		\widehat{w}(e), & \mathrm{otherwise.}
	\end{cases}
\end{equation}

\subsection{Optimal Path}

Define a path $\pi = (v_1, v_2 ,\ldots, v_m)$ on the graph $G=(V,E)$ as an ordered set of distinct vertices $v_i \in V$, $i = 1,\ldots,m$ such that, for any two consecutive vertices $v_i, v_{i+1}$, there exists an edge $e = (v_i, v_{i+1}) \in E.$ 
Throughout this paper, we will interchangeably denote a path as the sequence of such edges. 
With some abuse of notation, we denote the cost of a path as $w(\pi) \defeq \sum_{e\in \pi} w(e)$. 
Likewise, we denote $\overline{w}(\pi) \defeq \sum_{e\in \pi} \overline{w}(e)$ for the lazy cost estimate of the path $\pi$.
Let $v_\mathrm{s}, v_\mathrm{g} \in V$ be the start and goal vertices, respectively. 
Let $\Pi$ be the set of all paths from $v_\mathrm{s}$ to $v_\mathrm{g}$ in $G.$ 
Then, the shortest path planning problem seeks to find 
\begin{equation}
	\pi^* \defeq \argmin_{\pi \in \Pi}w(\pi).
\end{equation}

\subsection{Lazy LPA* Search Tree}
We maintain a lazy LPA* search tree to update the inconsistencies that arise from both graph changes and edge value discrepancies between the heuristic weight and the actual weight. 
The lazy LPA* search tree is identical to the standard LPA* search tree~\cite{Koenig2004}, except that lazy LPA* uses the lazy weight function $\overline{w}$ instead of the actual weight function $w$. 
For completeness of discussion, next we define the variables of the lazy LPA*. 

For each vertex, we store the two cost-to-come values, namely, the $g$-value and $rhs$-value to identify the inconsistent vertices, similarly to LPA*.
A vertex $v$ whose $g(v)=rhs(v)$ is called consistent, otherwise it is called inconsistent. An inconsistent vertex is locally overconsistent if $g(v)>rhs(v)$ and locally underconsistent if $g(v)<rhs(v).$
The $g$-value is the accumulated cost-to-come by traversing the previous search tree, whereas the $rhs$-value is the cost-to-come based on the $g$-value of the predecessor and the current $\overline{w}$-value of the current edge. 
Hence, the $rhs$-value is potentially better informed than the $g$-value, and it is defined as follows:
\begin{equation}
	rhs(v) \defeq 
	\begin{cases}
		0, & \mathrm{if}\; v=v_\mathrm{s}, \\
		\min_{u\in pred(v)} (g(u)+\overline{w}(u,v)), &
		\mathrm{otherwise.}
	\end{cases}
\end{equation}
Additionally, the $rhs$-value minimizing the predecessor of $v$ is stored as a backpointer, denoted with 
\begin{equation}
	bp(v)\defeq \argmin_{u\in pred(v)}(g(u)+\overline{w}(u,v)).
\end{equation}
Hence, the subpath from $v_\mathrm{s}$ to $v$ is retrieved by following the backpointers from $v$ to $v_\mathrm{s}$. 

The queue $Q$ prioritizes the inconsistent vertices using the key 
\begin{equation}
	k(v)=[\min (g(v),rhs(v))+h(v) \; ; \min (g(v),rhs(v))], 
\end{equation}
with lexicographic ordering, where $h(v)$ is a consistent heuristic cost-to-go from $v$ to $v_\mathrm{g}$.

\section{Lifelong-GLS Algorithm}

The proposed algorithm, Lifelong-GLS (L-GLS), consists of two loops: the inner loop and the outer loop. 
The inner loop is the main search loop which guarantees to return the shortest path in the current graph upon termination. 
The outer loop updates the current graph heuristically to reflect any external graph changes.
The edge evaluations in the inner loop may induce internal changes to the graph. Both external and internal changes are efficiently repaired by a lazy LPA* search tree. 

In the inner loop, the lazy LPA* search tree updates the new shortest path from $v_\mathrm{s}$ toward $v_\mathrm{g}$ in the current graph $G$ based on the previous search results. 
The lazily evaluated LPA* search tree uses the lazy estimates of the edge values when it propagates the inconsistencies to find the shortest subpath to the goal in the current graph. 
The first unevaluated edge on the shortest subpath returned by the lazy LPA* is then evaluated. 
If the evaluation results in inconsistency, then the lazy LPA* search tree is updated and returns the next best subpath for evaluation. 
If all the edges on the current shortest path to the goal returned by the lazy LPA* are already evaluated, then L-GLS has found the optimal solution and exits the inner loop.

In the outer loop, L-GLS waits for graph changes. When the edges of $G$ change, L-GLS assigns admissible heuristic values to the corresponding edges instead of evaluating them, to make sure that the lazy estimate of the path cost does not overestimate the optimal path cost.
Then, the inner loop begins again to search for the new optimal path. 
Hence, only a subset of the changed edges that could be on the shortest path in the current graph are actually evaluated. 

\subsection{Details of the Algorithm and Main Procedures}

\begin{algorithm}
	\caption{\textsf{Lifelong-GLS}($G, v_\mathrm{s}, v_\mathrm{g}$)}\label{lgls:a:lgls}
	\begin{algorithmic}[1]
		\Procedure{CalculateKey}{$v$} \Return
			\State {[$\min (g(v),rhs(v))+h(v) \; ; \min (g(v),rhs(v))$];}
		\EndProcedure
		\Procedure{UpdateVertex}{$v$}
			\If {$v\neq v_\mathrm{s}$}
				\State $bp(v) = \argmin_{u\in pred(v)} (g(u)+\overline{w}(u,v))$;
				\State $rhs(v) = g(bp(v)) + \overline{w}(bp(v),v)$;
			\EndIf
			\If {$v\in Q$} 
				$Q.\textsc{Remove}(v)$;
			\EndIf
			\If {$g(v)\neq rhs(v)$} 
				\State $Q.\textsc{Insert}((v,\textsc{CalculateKey}(v)))$; 
			\EndIf
		\EndProcedure
		\Procedure{ComputeShortestPath}{$\textsc{Event}$}
			\While{$Q.\textsc{TopKey} \prec \textsc{CalculateKey}(v_\mathrm{g})$ \Or \\
				$g(v_\mathrm{g})\neq rhs(v_\mathrm{g})$}
				\State $u=Q.\textsc{Pop}()$;\label{lgls:algo:line:vertexexpansion}
				\If {$g(u) > rhs(u)$}
					\State $g(u)=rhs(u)$; 
					\If {$\textsc{Event}(u)$ is triggered} \label{lgls:algo:line:expansionevent}
						\State \Return path from $v_\mathrm{s}$ to u;
					\EndIf
					\For {\textbf{all }$v\in succ(u)$}
						\textsc{UpdateVertex}($v$);
					\EndFor
				\Else
					\State $g(u)=\infty$;
					\For {\textbf{all }$v\in succ(u) \cup \set{u}$}
						\State \textsc{UpdateVertex}($v$);
					\EndFor
				\EndIf 
			\EndWhile
		\EndProcedure
		\Procedure{EvaluateEdges}{$\overline{\pi}$}
		\For{\textbf{each} $ e \in \overline{\pi}$}
		\If{$e \notin E_\mathrm{eval}$} 
		\State $\overline{w}(e) \gets w(e)$; 
		\State $E_\mathrm{eval}\gets E_\mathrm{eval} \cup \set{e}$;
		\If {$\overline{w}(e) \neq \widehat{w}(e) $} \Return $e$; 
		\EndIf
		\EndIf
		\EndFor 
		\EndProcedure
		\Procedure{Main}{$ $}
		\For {$\textbf{all } e \in E$} $\overline{w}(e) \gets \widehat{w}(e)$; 
		\EndFor
		\State $E_\mathrm{eval}\gets \varnothing$
		\State $rhs(v_\mathrm{s})=0$;
		\State $\textsc{UpdateVertex}(v_\mathrm{s})$;
		\While{true}
		\Repeat \label{lgls:algo:line:mainsearchbegin}
		\State $\overline{\pi} \gets \textsc{ComputeShortestPath}(\textsc{Event})$;
		\State $(u,v)\gets \textsc{EvaluateEdges}(\overline{\pi})$; \label{lgls:algo:line:evaluateedge}
		\State $\textsc{UpdateVertex}(v)$;
		\Until{$v_\mathrm{g} \in \overline{\pi}$ \textbf{and} $\overline{\pi} \subseteq E_\mathrm{eval}$} 
		\label{lgls:algo:line:mainsearchend}
		\State Wait for changes in $E$;
		\State $L \gets $the set of edges that changed;
		\For {$\textbf{all } e=(u,v) \in L$}
		\State $\overline{w}(e) \gets \widehat{w}(e)$; 
		\State $E_\mathrm{eval} \gets E_\mathrm{eval}\backslash \set{e}$;
		\State $\textsc{UpdateVertex}(v)$;
		\EndFor
		\EndWhile
		\EndProcedure
	\end{algorithmic}
\end{algorithm}

\begin{algorithm}
	\caption{Candidate $\textsc{Event}$ Definitions~\cite{Mandalika2019}}\label{lgls:a:events}
	\begin{algorithmic}[1]
		\Procedure{ShortestPath}{$v$}
		\If{$v=v_\mathrm{g}$} \Return true; \EndIf
		\EndProcedure
		\Procedure{ConstantDepth}{$v$, depth $\alpha$}
		\State $\overline{\pi} \gets$ path from $v_\mathrm{s}$ to $v$;
		\State $\alpha_v \gets $ number of unevaluated edges in $\overline{\pi}$;
		\If{$\alpha_v = \alpha $ \textbf{or} $v=v_\mathrm{g}$} \Return true; \EndIf
		\EndProcedure
	\end{algorithmic}
\end{algorithm}

Next, we describe the step-by-step procedure of L-GLS in greater detail. 
Before the first search begins, all $g$-values of the vertices are initialized with $\infty$ similar to the regular LPA*, and all lazy estimates of edge values are assigned with admissible heuristic values.
The first search begins by setting $rhs(v_\mathrm{s})= 0$ and inserting  $v_\mathrm{s}$ in the priority queue $Q$.
In the main search loop (Line~\ref{lgls:algo:line:mainsearchbegin}-\ref{lgls:algo:line:mainsearchend} of Algorithm~\ref{lgls:a:lgls}) the lazy LPA* search tree is grown with \textsc{ComputeShortestPath(Event)} until an \textsc{Event} is triggered by the expansion of a leaf vertex which just became consistent upon this expansion (Line~\ref{lgls:algo:line:expansionevent} of Algorithm~\ref{lgls:a:lgls}).
Then, the subpath to this leaf vertex which triggered the \textsc{Event} is returned for evaluation (Line~\ref{lgls:algo:line:evaluateedge} of Algorithm~\ref{lgls:a:lgls}). 
Then, \textsc{EvaluateEdges} evaluates the unevaluated edges along the subpath and updates the lazy estimates with their true weights. 
If the evaluation of an edge results in a different value than the previous lazy estimate, then \textsc{EvaluateEdges} returns the edge for the lazy LPA* to update this change accordingly by \textsc{UpdateVertex}. 
The inconsistency is propagated by the lazy LPA* again until the next time the \textsc{Event} is triggered. 
If the path to the goal is found, and all the edges along this path are evaluated in the current graph, then the path is indeed the optimal path in the current graph. This procedure repeats again when the graph changes. 

The procedure \textsc{UpdateVertex} is identical to that of the regular LPA*. The only difference is that when $\textsc{UpdateVertex}(v)$ is called, the $rhs$-value of the vertex $v$ is updated based on the lazy estimate of the incident edge values. 
This is done to avoid edge evaluations of the irrelevant incident edges of $v$. 
When a minimizing predecessor is found lazily, then the vertex assigns its backpointer to this predecessor. Finally, the key of this vertex is updated with \textsc{CalculateKey} to be prioritized in the queue $Q.$ 

The choice of an $\textsc{Event}$ function determines the balance between the vertex expansion (Line~\ref{lgls:algo:line:vertexexpansion} of Algorithm~\ref{lgls:a:lgls}) and the edge evaluation (Line~\ref{lgls:algo:line:evaluateedge} of Algorithm~\ref{lgls:a:lgls}), as in the GLS framework. 
For example, if one chooses the \textsc{ShortestPath} as the \textsc{Event}, then the algorithm becomes a version of Lifelong-LazySP~\cite{Dellin2016}. 
That is, the lazy LPA* repairs its inconsistent part of the tree all the way up to the goal, then returns the shortest path to the goal for evaluation. This minimizes the number of edge evaluations of the inner loop.
On the other hand, if one chooses the \textsc{ConstantDepth} of GLS as the \textsc{Event}, then the algorithm becomes a version of Lifelong-LRA*~\cite{Mandalika2018}. The tree repairing (vertex expansion) of the lazy LPA* is reduced, since the inconsistency propagation is restricted not to exceed a certain depth before evaluating the edges.
This comes at the expense of possibly more edge evaluations. Some candidate \textsc{Event} definitions of GLS~\cite{Mandalika2019} are reproduced in Algorithm~\ref{lgls:a:events}.

Note that the lazy LPA* algorithm maintained under the L-GLS algorithm is almost identical to the regular LPA* algorithm except at three points. 
First, the procedure \textsc{UpdateVertex} of L-GLS updates an inconsistent vertex with respect to the lazy estimate of the incident edges instead of the actual value. 
Second, \textsc{ComputeShortestPath} is identical to that of LPA* in the way it expands the inconsistent vertices of the lowest key first, that is, when it expands an inconsistent vertex, it makes an underconsistent vertex overconsistent and an overconsistent vertex consistent. 
The only difference is that when the overconsistent vertices are expanded, the \textsc{Event} checks whether to continue or stop propagating the inconsistency information to the successors. 
Finally, when the graph changes, L-GLS updates the changed edge values with admissible heuristic values lazily instead of evaluating them to find the exact values. 
Hence, the lazy LPA* inherits all the theoretical properties of the regular LPA*, albeit with respect to a different weight, namely $\overline{w}$ rather than $w$. 
This becomes useful in our analysis that proves the correctness of the algorithm. 

The L-GLS algorithm is different from the GLS algorithm in the following points. 
First, L-GLS stores the previous search results to propagate any inconsistencies efficiently in dynamic graphs, whereas GLS is explicitly designed for a shortest path planning problem in a static graph.\footnote{Although GLS may use an LPA* search tree in a similar way, the purpose of the LPA* search tree is to repair the search tree from the interior changes that come from revealing obstacles by the edge evaluations, rather than the exterior changes resulting from the environmental changes.}
L-GLS can possibly evaluate a much fewer number of edges compared to the GLS from scratch, since the search tree of L-GLS is better informed than that of GLS. 
Second, the exact values for all feasible edges are known a priori in the GLS framework, that is, the heuristic estimates of all feasible edges are accurate. The edge evaluation only reveals a binary trait of the edge, that is, whether the edge is feasible or not, rather than its exact cost. 
This is relaxed in L-GLS, such that the edge costs can vary upon the evaluation. This relaxation is important in problem domains where obtaining an accurate heuristic edge cost may be difficult.  
As long as the heuristic edge cost does not overestimate the actual edge cost, L-GLS finds the optimal solution in the current graph.

\section{Analysis}

We now present some of the properties of L-GLS to provide insights how it works. We also prove the completeness and correctness of the algorithm, based on the inherited properties from both the LPA* and the GLS algorithms. First, let us state two facts that are invariant during the main search loop.

\begin{restatable}{invariant}{LGLSLazyestimateInvariant}
	\label{lgls:invariant:LGLSLazyestimateInvariant}
	The lazy estimate of an edge never overestimates the true edge value, that is, $\overline{w} \leq w$. 
\end{restatable}
\begin{proof}
	Since $\overline{w}(e) = w(e)$ for all $e\in E_\mathrm{eval},$ and $\overline{w}(e) = \widehat{w}(e) \leq w(e)$ for all $e\notin E_\mathrm{eval},$ it follows that $\overline{w}(e)\leq w(e)$ for all $e\in E.$
\end{proof}

The next invariant shows that when the lazy LPA* returns the shortest subpath to the goal, then this subpath is optimal. 
This follows from the theoretical properties of LPA*, which is similar to A*. 
\begin{restatable}{invariant}{LGLSLPAInvariant}
	\label{lgls:invariant:LGLSLPAInvariant}
	The output subpath $\overline{\pi}$ from $v_\mathrm{s}$ to $v$ of \textsc{ComputeShortestPath}(\textsc{Event}) is optimal with respect to $\overline{w}$, that is,
	$\overline{\pi} =\argmin_{\pi \in \Pi_v} \overline{w}(\pi)$, where $\Pi_v$ is the set of paths from $v_\mathrm{s}$ to $v.$
\end{restatable}
\begin{proof}
	\textsc{ComputeShortestPath} with an $\textsc{Event}$ returns the path $\overline{\pi}$ from $v_\mathrm{s}$ to $v$, when the triggering vertex $v$ is expanded. Right before the expansion, $v$ was locally overconsistent.
	Theorem~6 of LPA* \cite{Koenig2004} states that whenever \textsc{ComputeShortestPath} selects a locally overconsistent vertex for expansion, then the g-value of $v$ is optimal with respect to $\overline{w}$.
\end{proof}

Now we show the completeness and correctness of the inner loop of L-GLS. The first theorem is due to the completeness of GLS~\cite{Mandalika2019}, which we restate here. 
\begin{restatable}{theorem}{LGLScomplete}
	\label{lgls:invariant:LGLScomplete}
	Let \textsc{Event} be a function that on halting ensures there is at least one unevaluated edge on the current shortest path or that the goal is reached. 
	Then, the inner loop (Line~\ref{lgls:algo:line:mainsearchbegin}-\ref{lgls:algo:line:mainsearchend}) of L-GLS implemented with \textsc{Event} on a finite graph terminates.
\end{restatable}
\begin{proof}
	Suppose the path to the goal has not been evaluated, such that \textsc{ComputeShortestPath(Event)} returns at least one unevaluated edge to evaluate. Since there is a finite number of edges, the inner loop will eventually terminate.
\end{proof}

\begin{restatable}{theorem}{LGLScorrect}
	\label{lgls:theorem:LGLScorrect}
	L-GLS finds the shortest path with respect to the current graph when the inner loop (Line~\ref{lgls:algo:line:mainsearchbegin}-\ref{lgls:algo:line:mainsearchend}) terminates.
\end{restatable}
\begin{proof}
	Let $\pi^*$ be the optimal path with respect to $w$ in the current graph, that is, $w(\pi^*)=\min_{\pi\in\Pi} w(\pi)$, where $\Pi$ is the set of all paths from $v_\mathrm{s}$ to $v_\mathrm{g}$. L-GLS terminates its inner-loop when $v_g \in \overline{\pi}$ and $\overline{\pi}\subseteq E_\mathrm{eval},$ where $\overline{\pi}$ is the output subpath of \textsc{ComputeShortestPath(Event)}.
	Then, we have 
	\begin{equation}
		\overline{w}(\overline{\pi}) = \sum_{e\in\overline{\pi}} \overline{w}(e)
		\leq \sum_{e'\in\pi^*} \overline{w}(e') 
		\leq \sum_{e'\in\pi^*} w(e') 
		= w(\pi^*),
	\end{equation}
	where the first inequality holds by Invariant~\ref{lgls:invariant:LGLSLPAInvariant}, and the second inequality follows by Invariant~\ref{lgls:invariant:LGLSLazyestimateInvariant}. 
	Hence, $\overline{w}(\overline{\pi}) \leq w(\pi^*),$ and since $\overline{\pi} \subseteq E_\mathrm{eval},$ we have $w(\overline{\pi}) = \overline{w}(\overline{\pi}) \leq w(\pi^*).$
	But $w(\pi^*) \leq w(\overline{\pi})$, since $\pi^*$ is the optimal path.
	Therefore, $\overline{\pi}$ must be the optimal path with respect to $w.$
\end{proof}


Note that the results of Theorem~\ref{lgls:theorem:LGLScorrect} can be extended to find a bounded suboptimal path if an inflated heuristic weight function is used instead of an admissible heuristic weight function.
In addition, if the output subpath of \textsc{ComputeShortestPath} is no longer than the optimal subpath by some factor, then the solution obtained by L-GLS is no longer than the optimal solution by the same factor. 
This is formalized in Theorem~\ref{lgls:theorem:LGLSbounded} below.
\begin{restatable}{theorem}{LGLSbounded}
	\label{lgls:theorem:LGLSbounded}
	Assume $\widehat{w}\leq \varepsilon_1 w$, and assume that the output subpath $\overline{\pi}$ of \textsc{ComputeShortestPath} from $v_\mathrm{s}$ to $v$ satisfies $\overline{w}(\overline{\pi}) \leq \varepsilon_2 \min_{\pi \in \Pi_v}\overline{w}(\pi)$, where $\varepsilon_1, \varepsilon_2 \geq 1$. Then, when the inner loop (Line~\ref{lgls:algo:line:mainsearchbegin}-\ref{lgls:algo:line:mainsearchend}) terminates,
	L-GLS finds a bounded suboptimal path $\pi$ from $v_\mathrm{s}$ to $v_\mathrm{g}$ such that $w(\pi) \leq \varepsilon_1 \varepsilon_2 w(\pi^*),$
	where $\pi^*$ is the optimal path in the current graph.
\end{restatable}
\begin{proof}
	Since $\widehat{w}\leq \varepsilon_1 w$, we have $\overline{w}(e) \leq \varepsilon_1 w(e)$ for all $e\in E.$
	Recall that L-GLS terminates its inner-loop when $v_g \in \overline{\pi} \subseteq E_\mathrm{eval}.$
	Then, we have 
	\begin{equation}
		\begin{aligned}
			\overline{w}(\overline{\pi}) = \sum_{e\in\overline{\pi}} \overline{w}(e)
			& \leq \varepsilon_2\sum_{e'\in\pi^*} \overline{w}(e') \\
			& \leq \varepsilon_2\sum_{e'\in\pi^*} \varepsilon_1 w(e') 
			= \varepsilon_1 \varepsilon_2 w(\pi^*).
		\end{aligned}
	\end{equation}
	Hence, $\overline{w}(\overline{\pi}) \leq \varepsilon_1 \varepsilon_2 w(\pi^*)$ and since $\overline{\pi} \subseteq E_\mathrm{eval},$ we have $w(\overline{\pi}) = \overline{w}(\overline{\pi}) \leq \varepsilon_1 \varepsilon_2 w(\pi^*).$ Therefore, the path length of $\overline{\pi}$ must not be greater than the optimal path by a factor $\varepsilon_1 \varepsilon_2.$
\end{proof}

The inflated heuristic weight function biases the search greedily, and often a high inflation factor $\varepsilon_1$ helps finding a solution faster. The truncation factor $\varepsilon_2$ determines how early the inconsistent propagation of LPA* can be terminated, such that an existing path without further rewiring is already guaranteed not to exceed the optimal path by the factor $\varepsilon_2$ in length~\cite{Aine2016}. 
The two factors can be completely decoupled, but they have the same goal. They make the lazy search tree find a good enough solution fast for evaluation, instead of spending time to find the lazily evaluated optimal path, which is likely to be repaired anyways.

\section{Numerical Results}

In this section, we present numerical results comparing L-GLS to LPA* and GLS to demonstrate the efficiency of L-GLS in scenarios where the shortest path planning problem is solved consecutively in a dynamic environment. 
The search is performed on the same graph with evenly distributed vertices, in which two vertices are adjacent if they are within a predefined radius. 
The graph topology does not change throughout the experiment, and only the edge values change due to underlying environment changes. 
We present search results of path planning problems in $\R^2$ for the sake of visualization, and then we present search results of piano movers' problems in $\R^3$ and of manipulation problems in $\R^7$ using PR2, a mobile robot with 7D arms. 

\begin{figure*}[ht]
	\centering
	\begin{subfigure}{\myMSFigureScale\textwidth}
		\includegraphics[width=\myLineScale\linewidth]{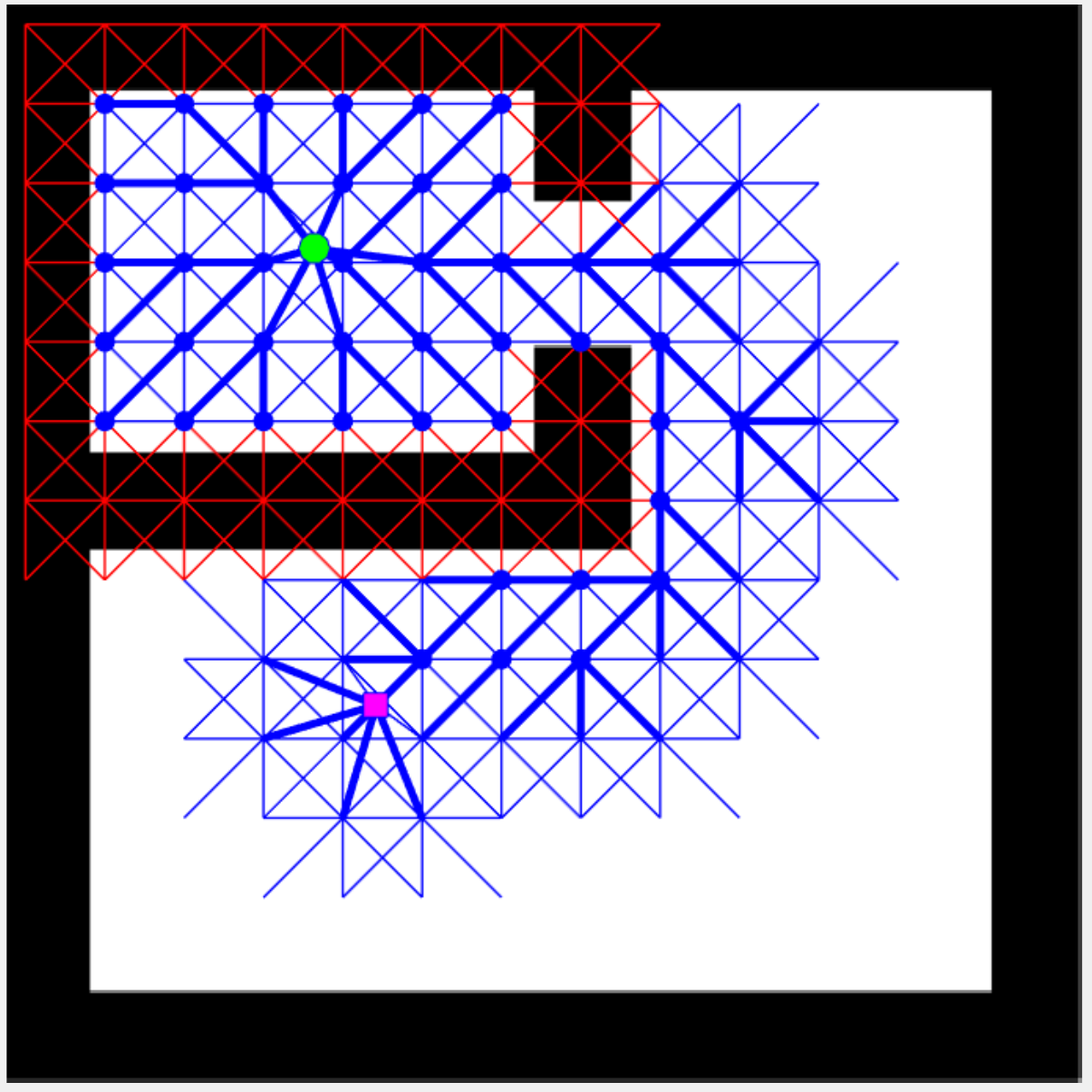}
	\end{subfigure}
	\begin{subfigure}{\myMSFigureScale\textwidth}
		\includegraphics[width=\myLineScale\linewidth]{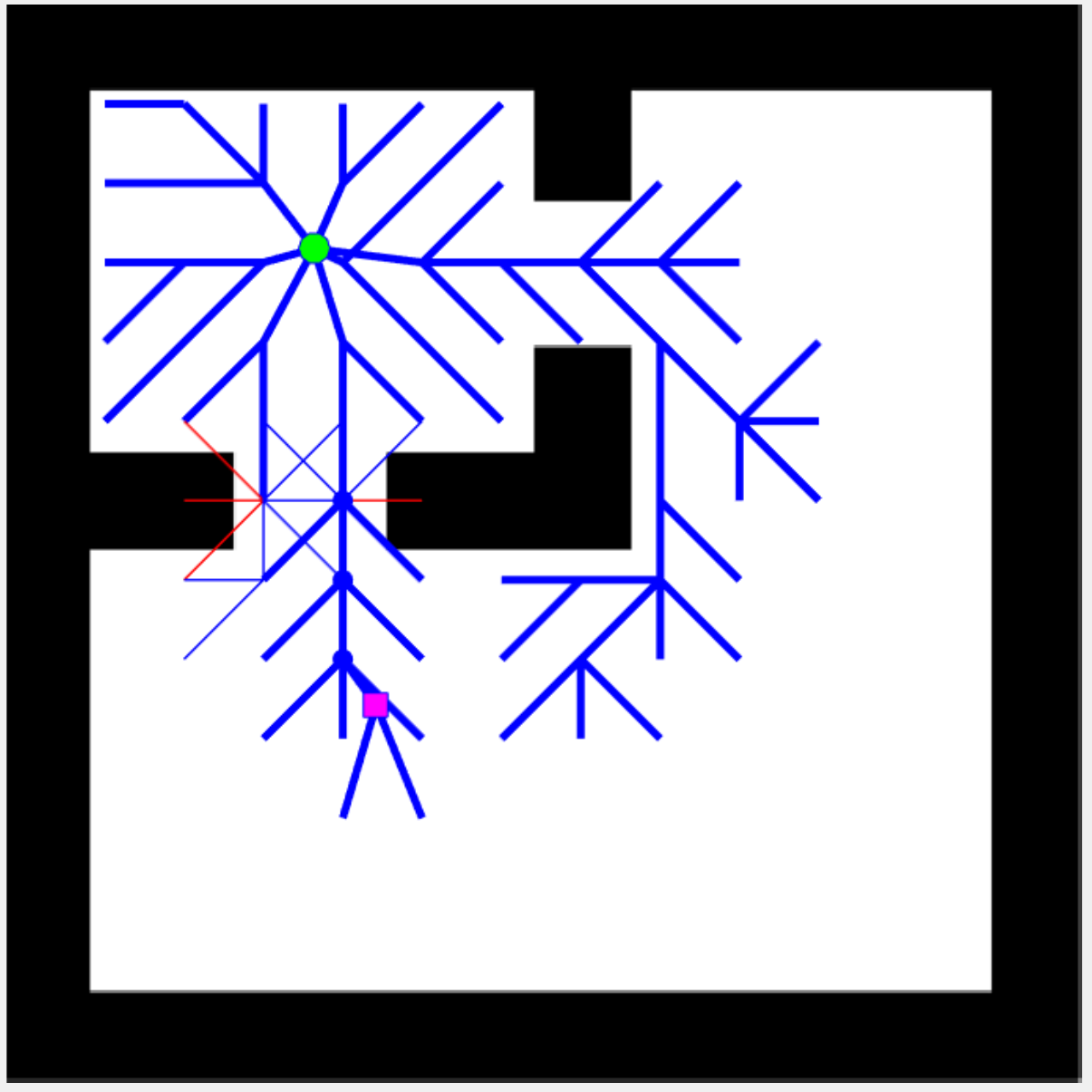}
	\end{subfigure}
	\begin{subfigure}{\myMSFigureScale\textwidth}
		\includegraphics[width=\myLineScale\linewidth]{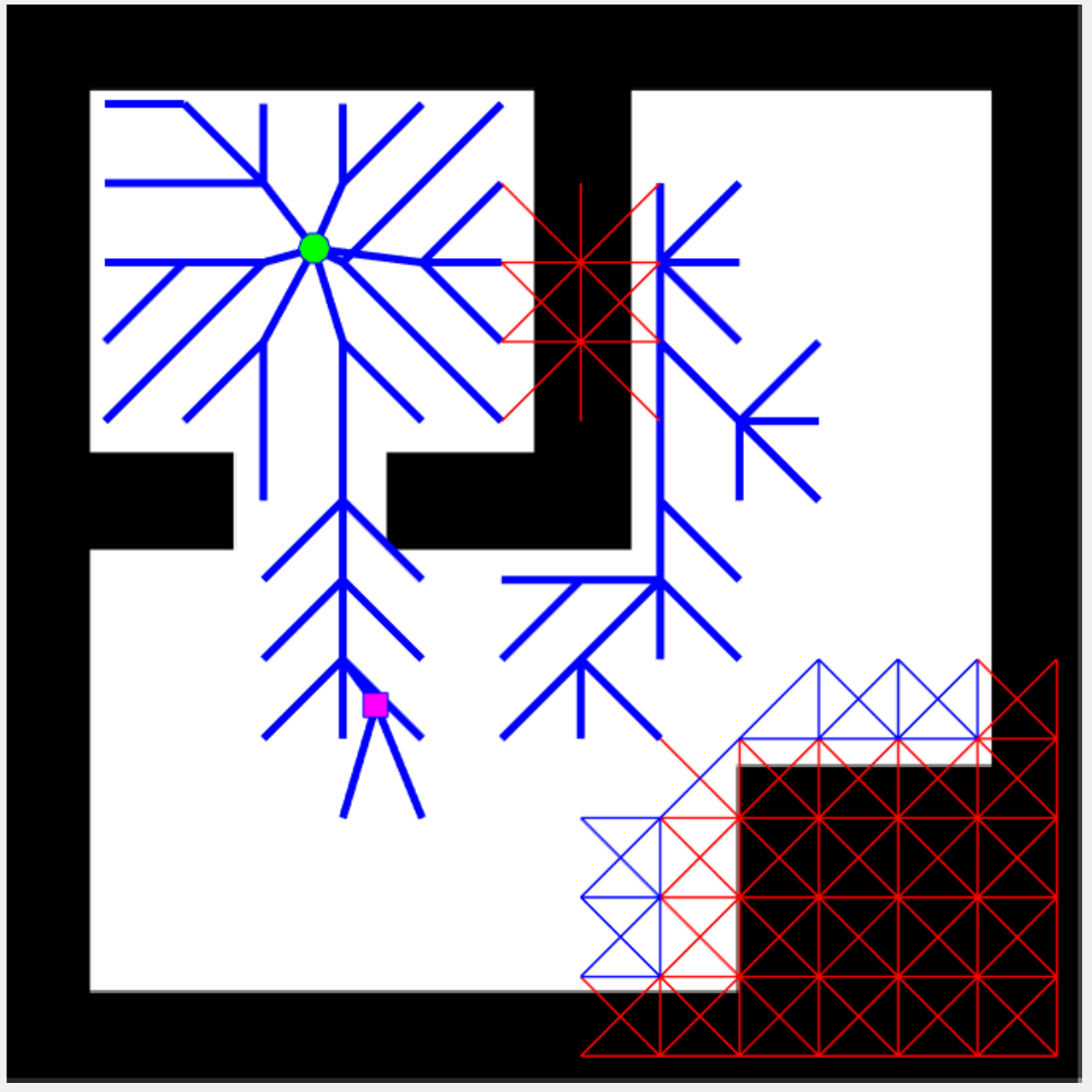}
	\end{subfigure}
	\begin{subfigure}{\myMSFigureScale\textwidth}
		\includegraphics[width=\myLineScale\linewidth]{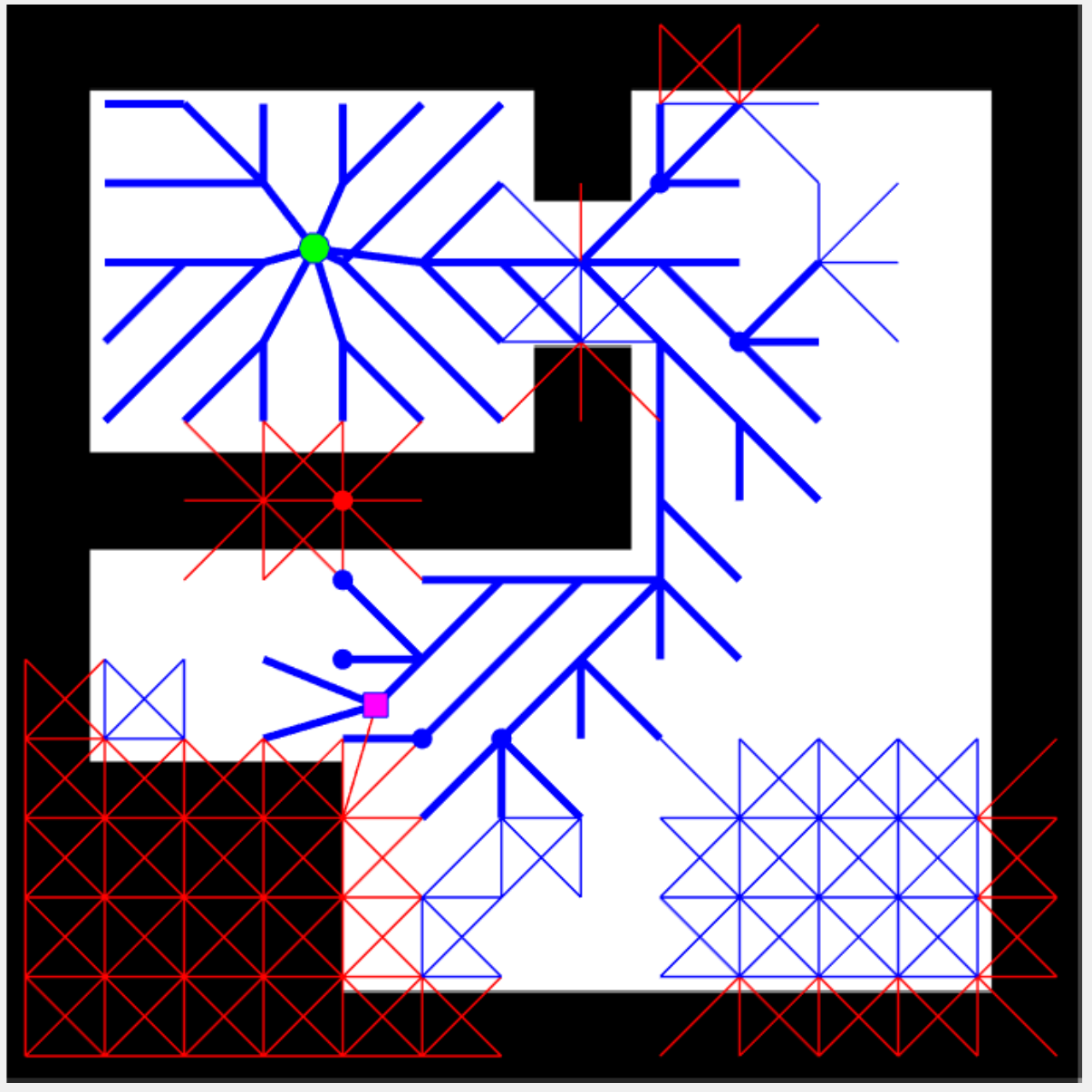}
	\end{subfigure}
	\begin{subfigure}{\myMSFigureScale\textwidth}
		\includegraphics[width=\myLineScale\linewidth]{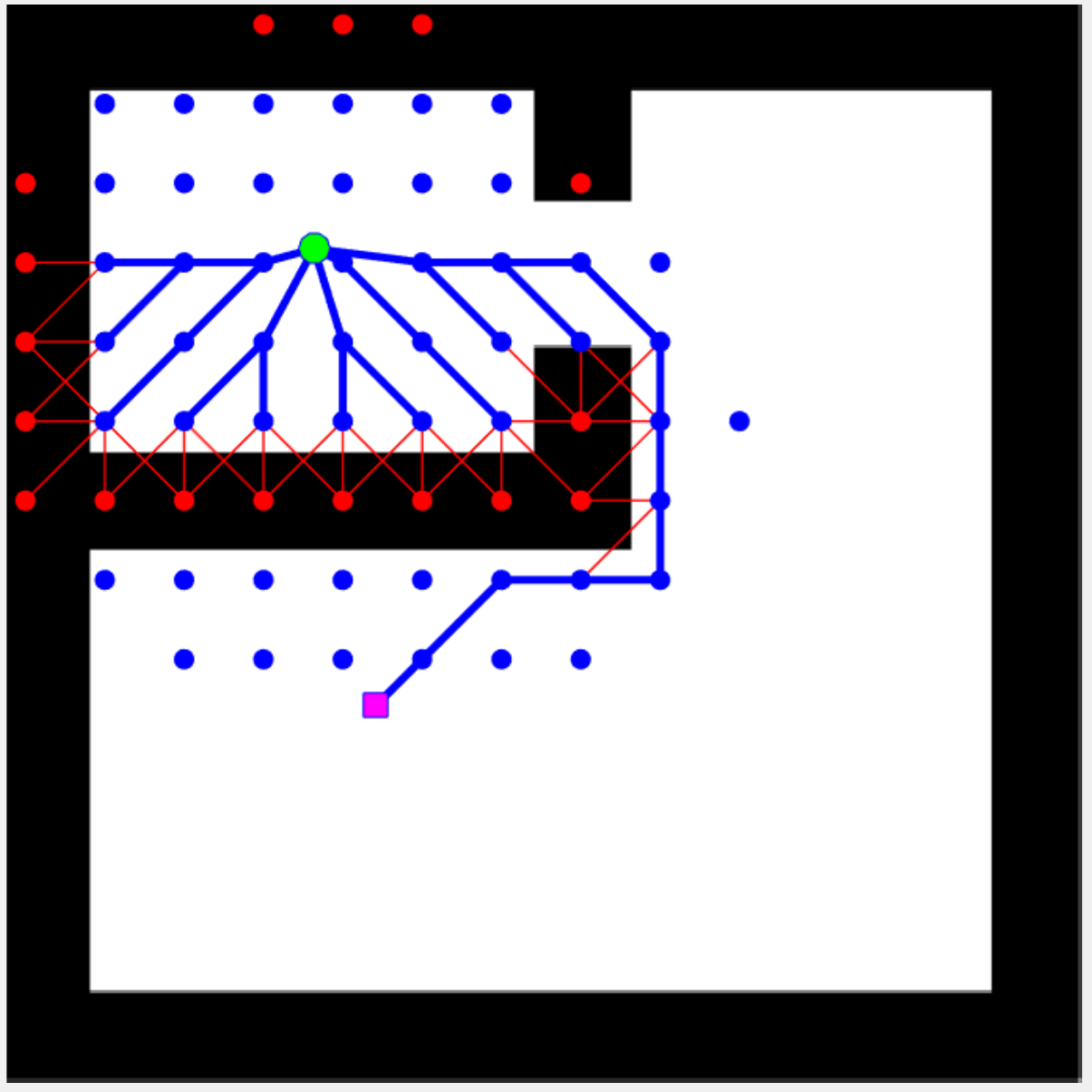}
		\caption{}
	\end{subfigure}
	\begin{subfigure}{\myMSFigureScale\textwidth}
		\includegraphics[width=\myLineScale\linewidth]{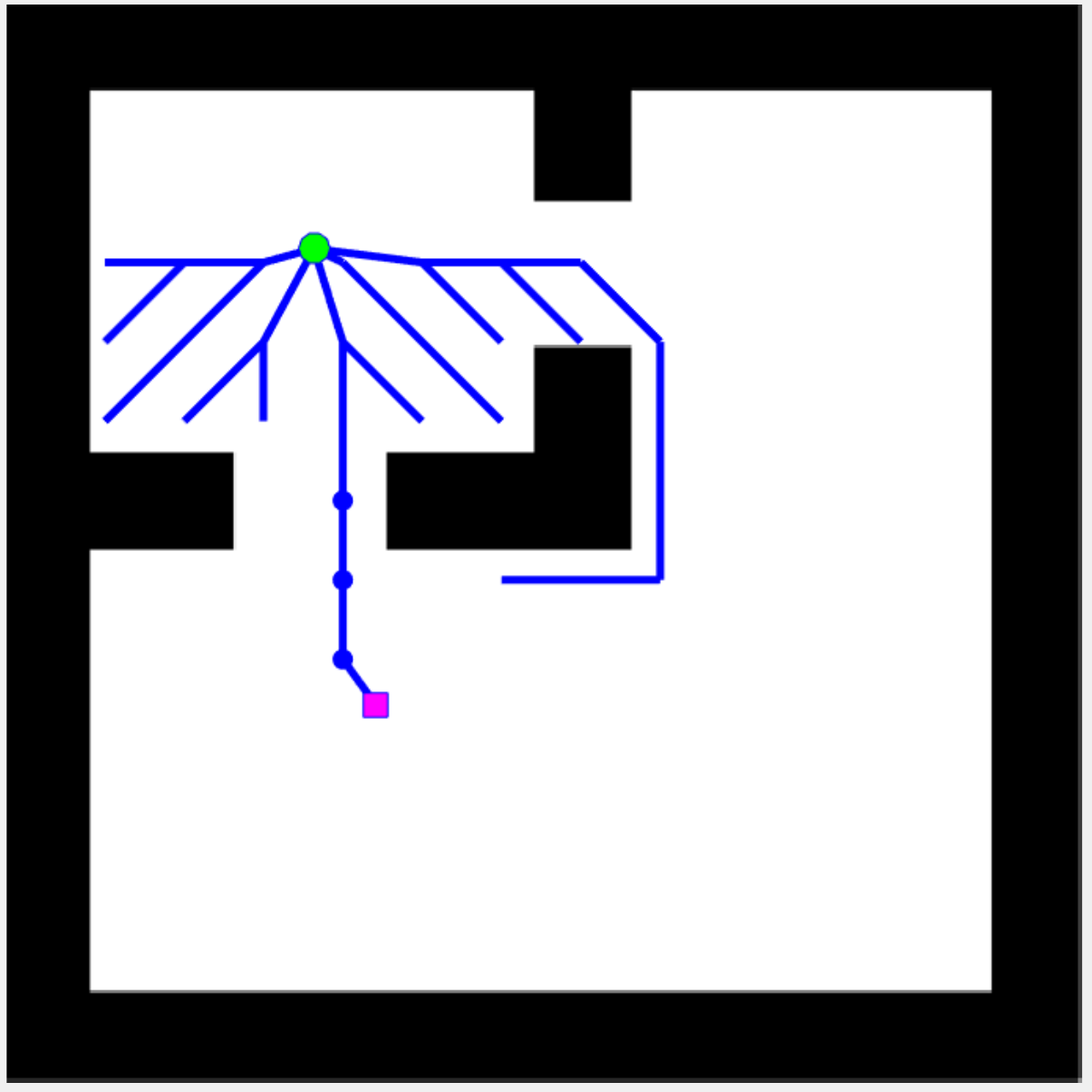}
		\caption{}
	\end{subfigure}
	\begin{subfigure}{\myMSFigureScale\textwidth}
		\includegraphics[width=\myLineScale\linewidth]{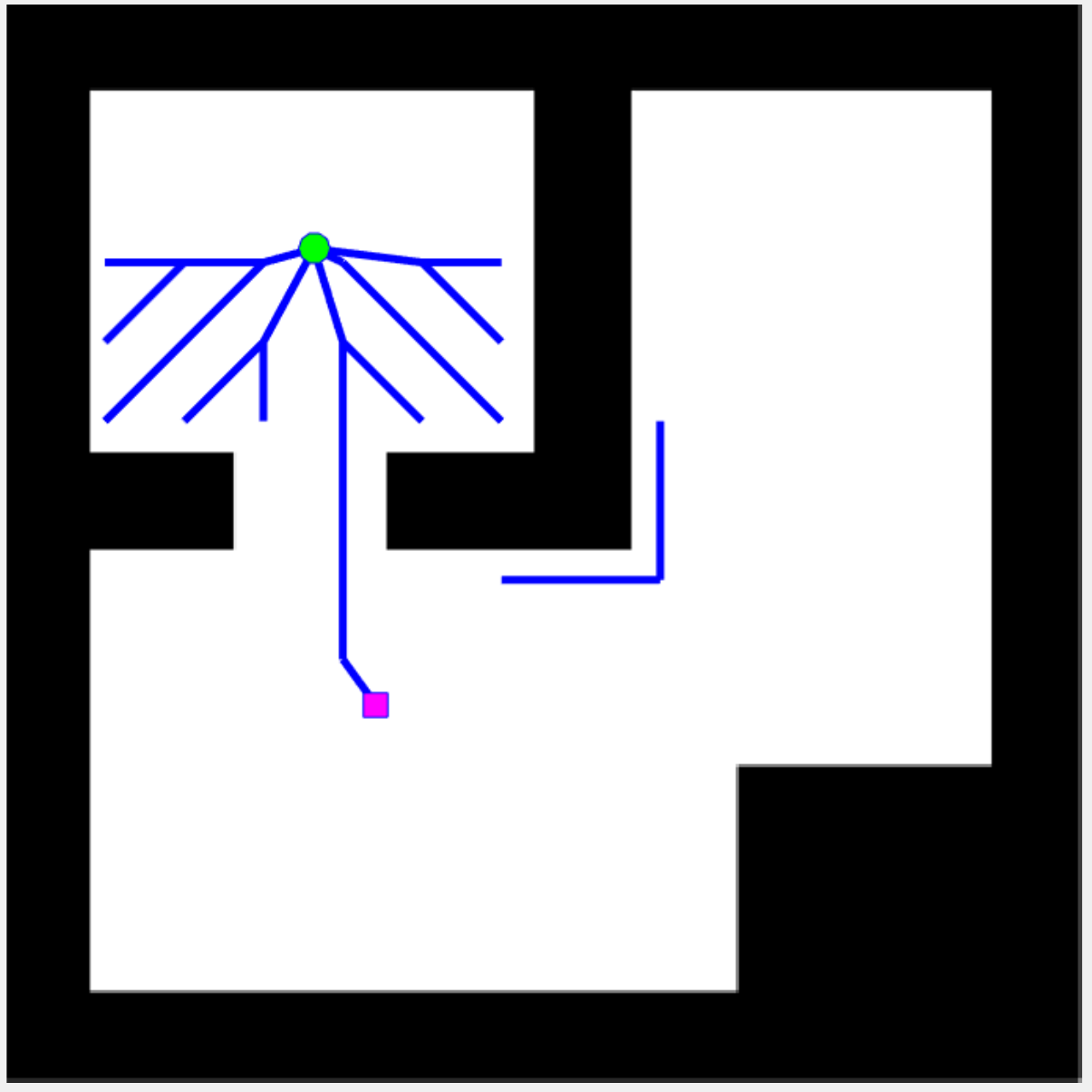}
		\caption{}
	\end{subfigure}
	\begin{subfigure}{\myMSFigureScale\textwidth}
		\includegraphics[width=\myLineScale\linewidth]{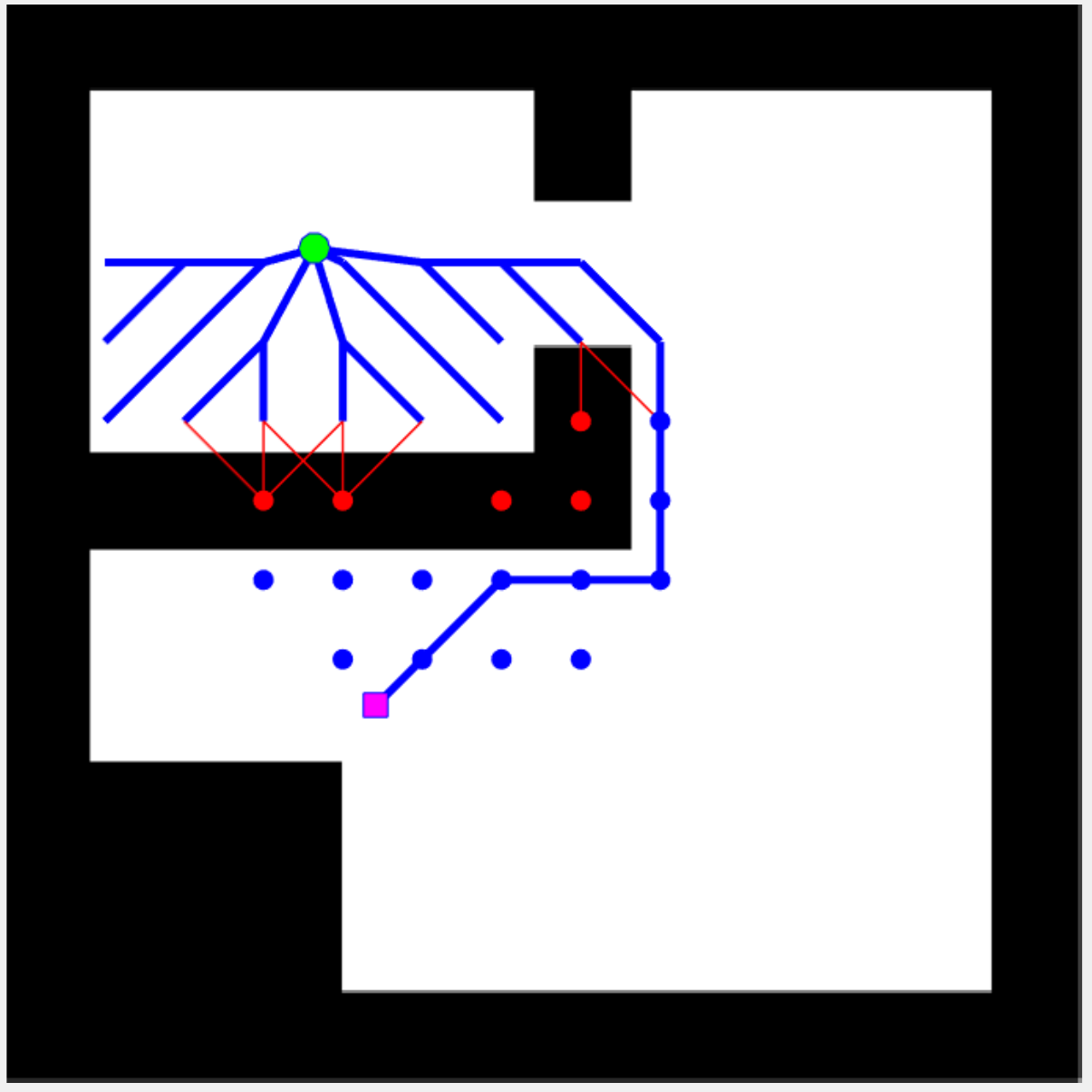}
		\caption{}
	\end{subfigure}
	\caption{LPA*(top row) and Lifelong-LazySP(bottom row) search results to find the shortest path from start vertex(\tikzcircle[blue, fill=green]{2.5pt}) to goal vertex(\tikzsquare[blue, fill=magenta]{4.5pt}) per environment change, from left to right: (a) first search, (b) second search, (c) third search, and (d) final search. 
		Lines(\tikzline[blue,semithick]{}\tikzline[red,semithick]{}) are the evaluated edges, and dots(\tikzcircle[blue, fill=blue]{2.0pt}\,\tikzcircle[red, fill=red]{2.0pt}) are the expanded vertices during the current search. Bold lines(\tikzline[blue,very thick]{}) are the edges belonging to the current search tree. Blue and red represents free and obstacle, respectively.}
	\label{lgls:f:result_example}
\end{figure*}

During the 2D experiments, the environment changed three times after the shortest path was found in each of the changed environments (see Figure~\ref{lgls:f:result_example}). 
We recorded the number of vertex expansions and the number of edge evaluations in each search episode for the three algorithms: LPA*, L-GLS, and GLS for each search (see Figure~\ref{lgls:f:result2d}).
We chose \textsc{ShortestPath} for the \textsc{Event} function for both L-GLS and GLS. 
Hence, L-GLS and GLS were equivalent to Lifelong-LazySP and LazySP, respectively.

\begin{figure*}[ht]
	\centering
	\begin{subfigure}{0.45\textwidth}
		\includegraphics[width=\myLineScale\textwidth]{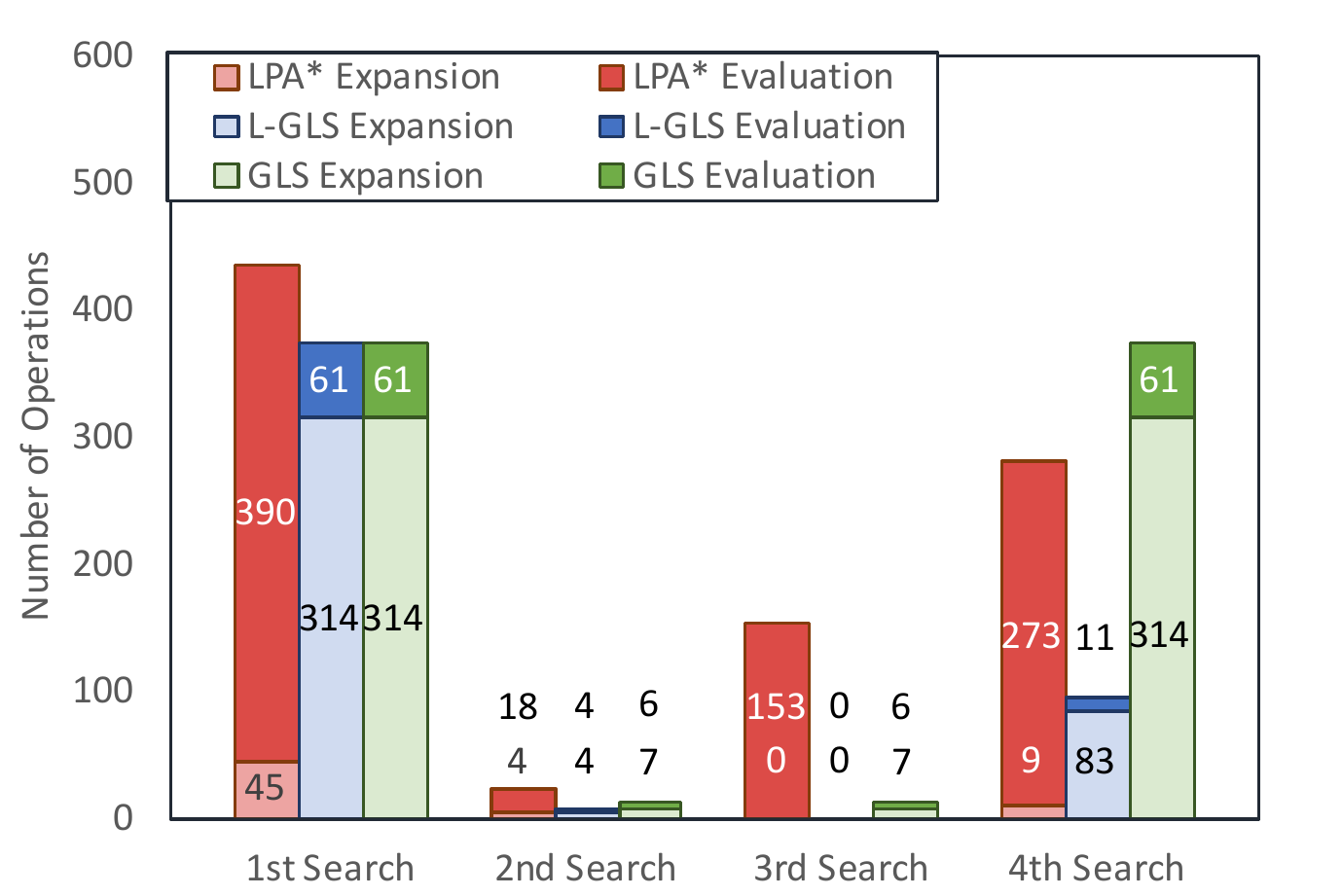}
		\caption{}
	\end{subfigure}
	\begin{subfigure}{0.45\textwidth}
		\includegraphics[width=\myLineScale\linewidth]{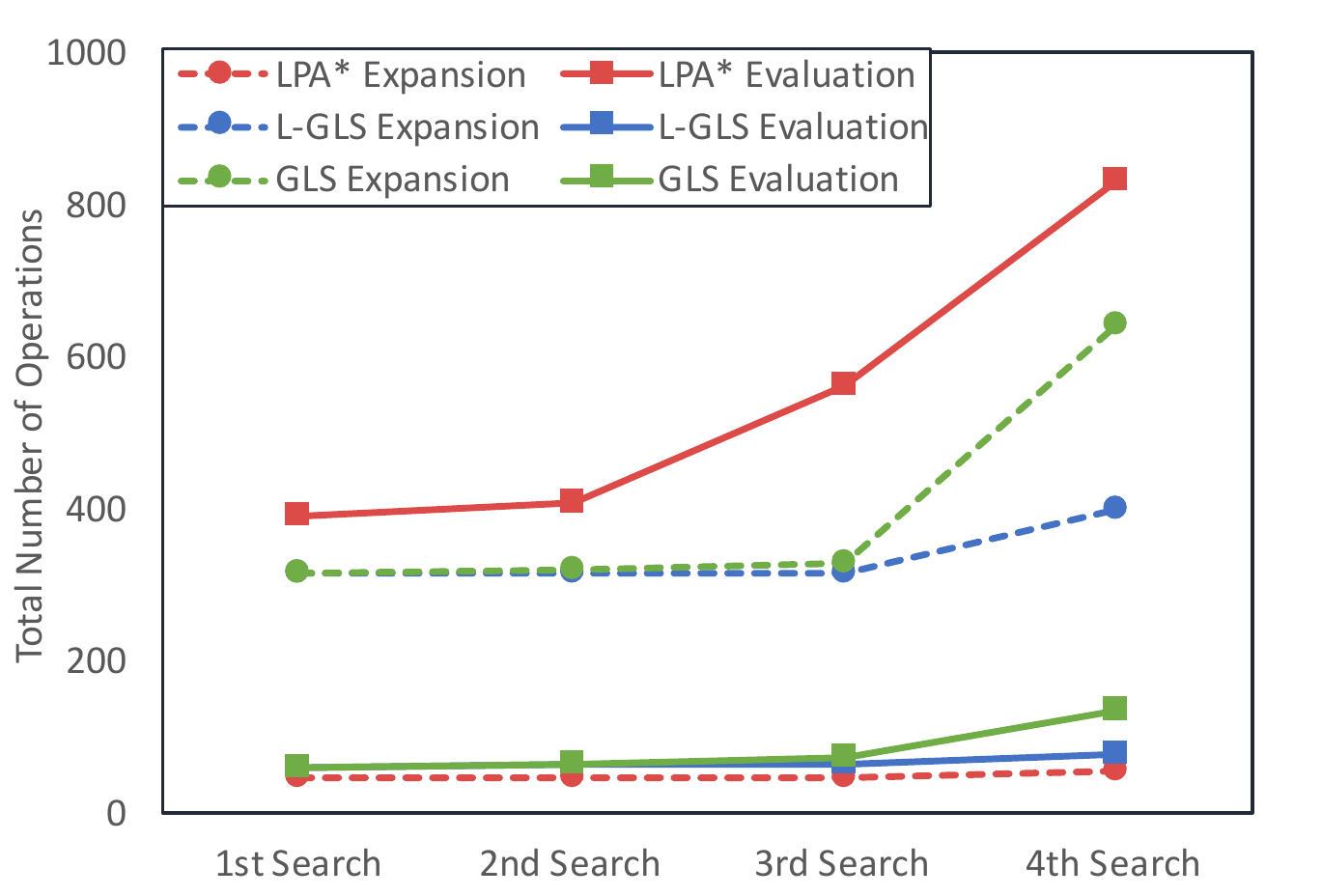}
		\caption{}
	\end{subfigure}\hfill
	\caption{Number of edge evaluations and vertex expansions for LPA*, L-GLS, and GLS in the four consecutive environments (a) for each search episode; (b) for accumulated results}
	\label{lgls:f:result2d}
\end{figure*}

In the first search, LPA* is equivalent to A*, and L-GLS is equivalent to GLS (See Figure~\ref{lgls:f:result_example}.a). LPA* evaluated 390 edges and expanded 45 vertices, whereas L-GLS and GLS both evaluated 61 edges and expanded 314 vertices.

After the first search, only a small part of the environment changed (see Figure~\ref{lgls:f:result_example}.b), opening a shorter passage to the goal. LPA* evaluated 18 edges corresponding to the change, then expanded 4 inconsistent vertices to find the shortest path in the current graph. L-GLS evaluated 4 edges that belong to the new shortest path to the goal, and expanded 4 inconsistent vertices. The GLS evaluated 7 edges and expanded 6 inconsistent vertices. 

When the environment changed in the irrelevant region (see Figure~\ref{lgls:f:result_example}.c), LPA* evaluated 153 edges corresponding to the environment change, but did not expand any vertices, as they were irrelevant to the current search. 
L-GLS did not do any additional operations to find the shortest path, since the path was already optimal. GLS was identical to the previous search with 7 edge evaluations and 6 vertex expansions.

Finally, the environment changed back to the first search episode with the addition to a new obstacle in the irrelevant region. The GLS search was identical to the first search episode with 61 edge evaluations and 314 vertex expansions. On the other hand, L-GLS evaluated only 11 edges and expanded 83 vertices. This is because the majority of the relevant edges were already evaluated during the previous searches, and the majority of the relevant vertices were already consistent. Similarly, LPA* expanded a fewer number of vertices and evaluated a fewer number of edges compared to the first search episode with 273 edge evaluations and 9 vertex expansions, since it utilized the previous search results. These results are illustrated in Figure~\ref{lgls:f:result2d}.

We also implemented LPA* and L-GLS as an OMPL Planner~\cite{Sucan2012} with the 
MoveIt! interface~\cite{Coleman2014} for the 3D piano movers' problem and for the 7D manipulator experiment. 
All the algorithm implementations were in C++, and the experiments were run on an 2.20 GHz Intel(R) Core(TM) i7-8750H CPU Ubuntu 16.04 LTS machine with 15.5GB of RAM. 

We find the shortest paths for the piano from the Apartment scenario in OMPL~\cite{Sucan2012} from a start configruation to a goal configuration without colliding with the moving obstacles (see Figure~\ref{lgls:f:piano}).  
There were three consecutive searches in the environment, where the first search was on scene 1 (Figure~\ref{lgls:f:piano} (a)), the second search was on scene 2 (Figure~\ref{lgls:f:piano} (b)), and the third search was on scene 3 (Figure~\ref{lgls:f:piano} (c)). 
The search was performed on a prebuilt graph with 8,000 vertices and 34,327 edges. The vertices were sampled using a Halton sequence in $\R^3$. 

Similarly, we find the shortest paths for the right arm of PR2 robot from a start configuration to a goal configuration without collision in a dynamic environment where the obstacle moves (see Figure~\ref{lgls:f:pr2}). There were three consecutive searches in the environment, where the first search was on scene 1 (Figure~\ref{lgls:f:pr2} (a)), the second search was on scene 2 (Figure~\ref{lgls:f:pr2} (b)), and the third search was on scene 1 again. 
The search was performed on a prebuilt graph with 30,000 vertices and 168,795 edges. The vertices were sampled using a Halton sequence in $\R^7$, bounded by the PR2 arm's joint-angle bounds. 
Two vertices are adjacent in this graph if the Euclidean distance between them is less than 0.9 rad. 

We compared five different planners including LPA*, L-GLS with infinite-lookahead, GLS with infinite-lookahead, L-GLS with one-step lookahead, and GLS with one-step lookahead.
The number of edge evaluations and the number of vertex expansions along with the approximated planning time are recorded for each search episode and tabulated in Table~\ref{lgls:t:pr2}. 
The approximate planning time was computed as the weighted sum of the number of edge evaluations and the number of vertex expansions. 
For the Piano Movers' problem, an edge evaluation took 0.20~ms on average and a vertex expansion took 0.86~ms on average. 
For the manipulation problem, an edge evaluation took 0.57~ms on average and a vertex expansion took 0.34~ms on average. 

\begin{figure*}[ht]
	\centering
	\begin{subfigure}{0.30\textwidth}
		\includegraphics[width=\myLineScale\linewidth]{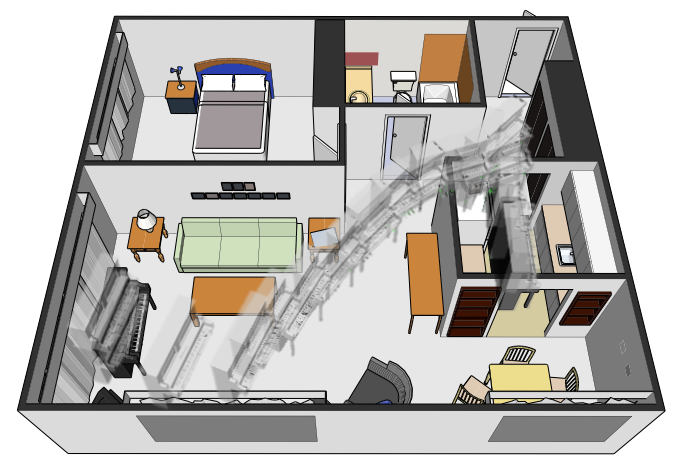}
		\caption{Scene 1}
	\end{subfigure}
	\begin{subfigure}{0.30\textwidth}
		\includegraphics[width=\myLineScale\linewidth]{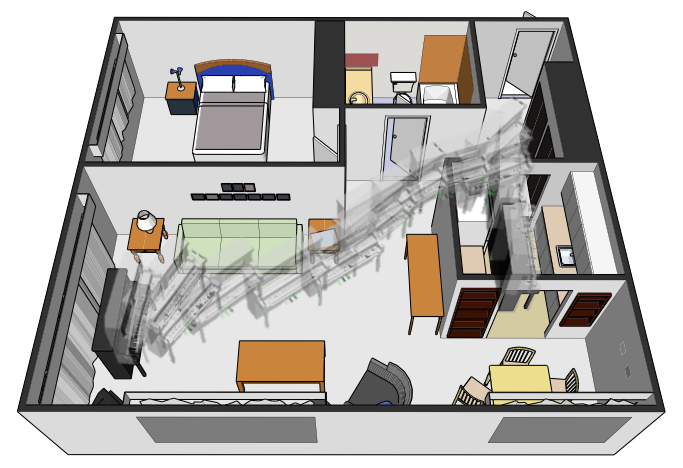}
		\caption{Scene 2}
	\end{subfigure}
	\begin{subfigure}{0.30\textwidth}
		\includegraphics[width=\myLineScale\linewidth]{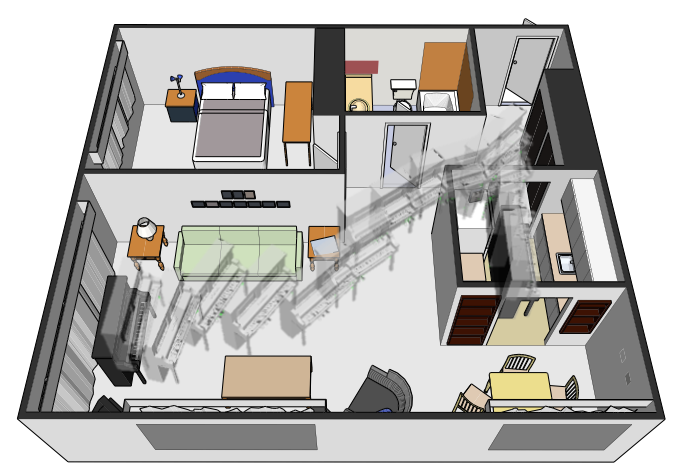}
		\caption{Scene 3}
	\end{subfigure}
	\caption{The shortest paths of the Piano Movers' problems in dynamic environment.}
	\label{lgls:f:piano}
\end{figure*}
\begin{figure*}[ht]
	\centering
	\begin{subfigure}{0.45\textwidth}
		\includegraphics[width=\myLineScale\linewidth]{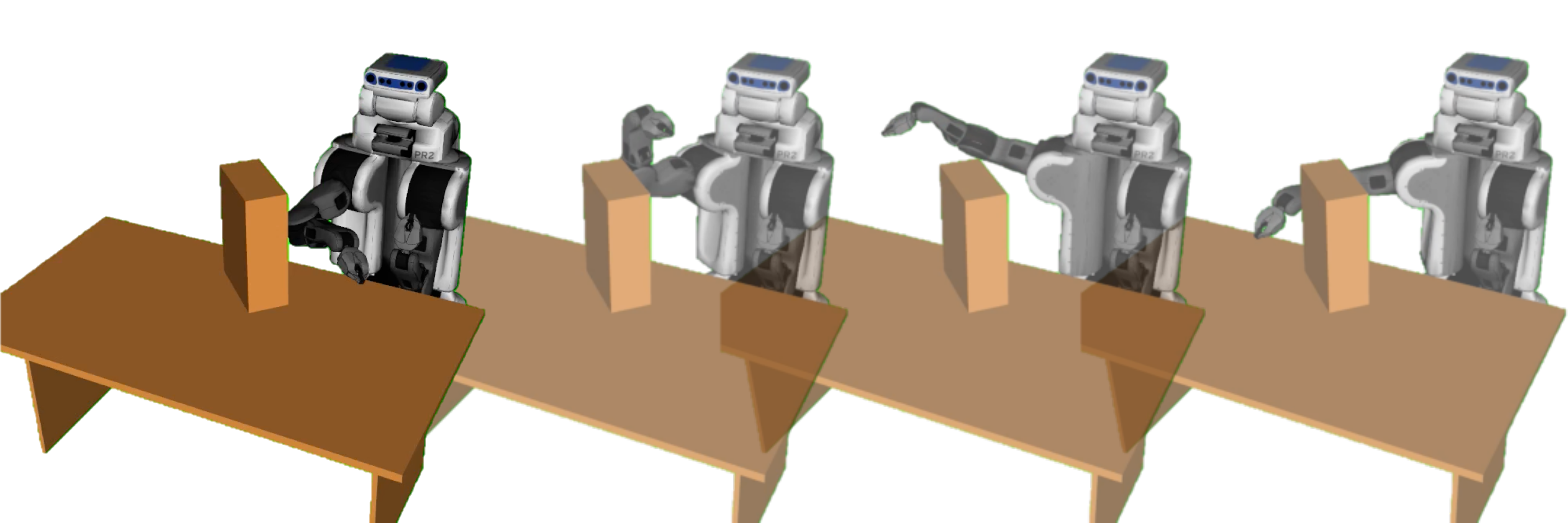}
		\caption{Scene 1}
	\end{subfigure}
	\begin{subfigure}{0.45\textwidth}
		\includegraphics[width=\myLineScale\linewidth]{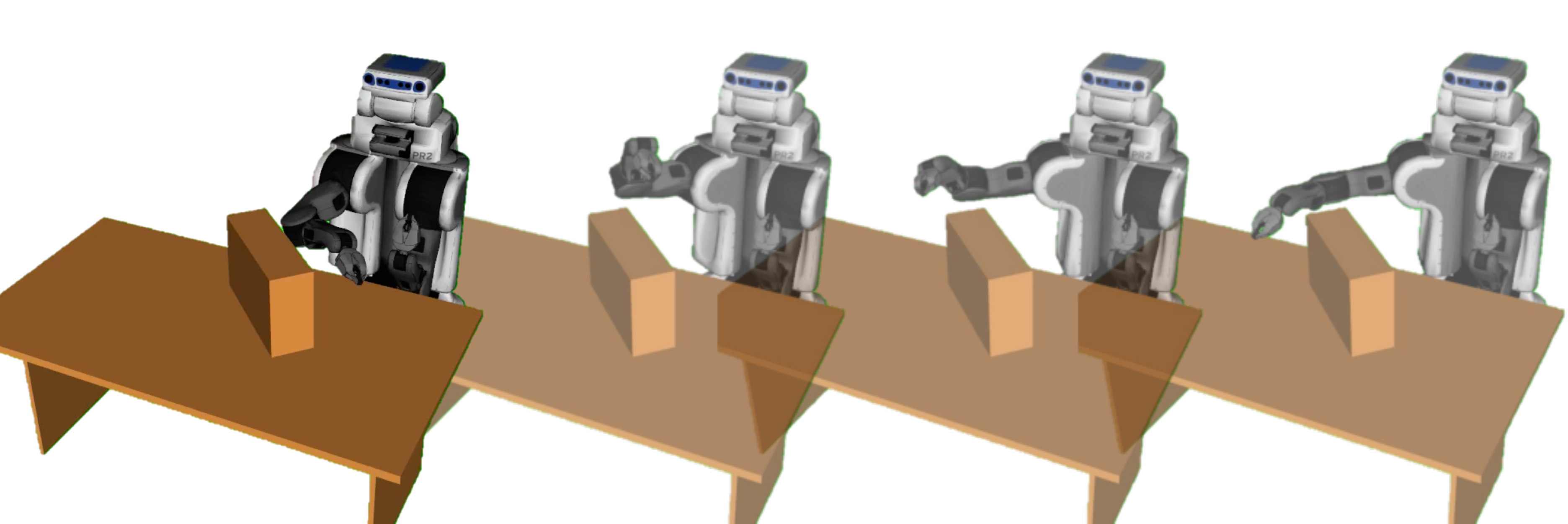}
		\caption{Scene 2}
	\end{subfigure}\hfill
	\caption{The shortest paths of the right arm of PR2 robot for the same query in dynamic environment.}
	\label{lgls:f:pr2}
\end{figure*}

\begin{table*}
	\centering
	\begin{small}
	\begin{tabular}{lccccc}
		\toprule
		\textbf{Piano Movers}      & LPA* & L-GLS($\infty$) & GLS($\infty$) & L-GLS($1$) & GLS($1$) \\ 
		\midrule
		\textbf{First Query} in scene 1    &          &       &          &        &       \\
		\# Edge Evaluation  & 24445 & 1524 & 1524 & 1867 & 1867 \\ 
		\# Vertex Expansion & 124 & 58649 & 58649 & 3612 & 3612 \\
		Total Time (s) & 5.08 & 50.6 & 50.6 & 3.48 & 3.48  \\
		&          &       &          &        &       \\
		\textbf{Second Query} in scene 2    &          &       &          &        &       \\
		\# Edge Evaluation  & 46384 & 46 & 952 & 67 & 1173 \\ 
		\# Vertex Expansion & 31 & 719 & 45085 & 277 & 2266 \\
		Total Time (s) & 9.46 & 0.626 & 38.8 & 0.251 & 2.18 \\		         
		&          &       &          &        &       \\
		\textbf{Third Query} in scene 3    &          &       &          &        &       \\
		\# Edge Evaluation  & 32312 & 37 & 702 & 68 & 872 \\ 
		\# Vertex Expansion & 23 & 379 & 37735 & 115 & 1671 \\
		Total Time (s) & 6.59 & 0.332 & 3.25 & 0.112 & 1.61 \\
		\midrule
		\textbf{PR2}     &  & &  & & \\ 
		\midrule
		\textbf{First Query} in scene 1    &          &       &          &        &       \\
		\# Edge Evaluation  & 10709 & 332 & 332 & 879 & 879 \\ 
		\# Vertex Expansion & 205 & 5427 & 5427 & 1555 & 1555 \\
		Total Time (s) & 6.25 & 2.04 & 2.04 & 1.04 & 1.04  \\
		&          &       &          &        &       \\
		\textbf{Second Query} in scene 2    &          &       &          &        &       \\
		\# Edge Evaluation  & 49251 & 7 & 18 & 20 & 81 \\ 
		\# Vertex Expansion & 9 & 37 & 155 & 33 & 139 \\
		Total Time (s) & 28.4 & 0.017 & 0.063 & 0.023 & 0.094 \\		         
		&          &       &          &        &       \\
		\textbf{Third Query} in scene 1    &          &       &          &        &       \\
		\# Edge Evaluation  & 13024 & 52 & 332 & 147 & 879 \\ 
		\# Vertex Expansion & 195 & 718 & 5427 & 363 & 1555 \\
		Total Time (s) & 7.58 & 0.275 & 2.04 & 0.209 & 1.04 \\		     				     		    
		\bottomrule
	\end{tabular}
	\end{small}
	\caption{Number of edge evaluations, number of vertex expansions, and approximated planning time for different planners over three consecutive search queries in a dynamic environment.}
	\label{lgls:t:pr2}
\end{table*}

\section{Conclusion}

We have presented a new replanning algorithm to find the shortest path in a given graph efficiently using previous search results. The proposed algorithm maintains a lazy LPA* tree to efficiently repair the inconsistency of the existing search that arises either from external environment changes or internal discrepancies between the lazy estimate and the real weight of an edge cost. 
Based on the efficiency of LPA*, the propagation of vertex rewiring to repair any vertex inconsistencies is restricted only to the shortest path candidate.
Similar to the GLS framework, only the edges in the current shortest path candidate are evaluated. 
The proposed algorithm reduces by a substantial amount the edge evaluations per search compared to LPA*, and it can find a new shortest path significantly faster than  GLS, given a change in the graph. 
The completeness and correctness of the proposed algorithm are shown.
We prove that the Lifelong-GLS algorithm returns a solution that is no longer than the optimal solution by the product of two factors, namely the heuristic inflation factor and the truncation factor. 
Numerical simulations demonstrate the efficiency of the proposed algorithm compared to both LPA* and GLS in a dynamically changing environment.

 \section*{Acknowledgement}
 We would like to thank Aditya Mandalika for his help on setting up the benchmark testing for the piano movers' problem. 
 This work has been supported by ARL under DCIST CRA W911NF-17-2-0181 and SARA CRA W911NF-20-2-0095 and NSF under award IIS-2008686.

\bibliographystyle{plainnat}
\bibliography{bib/references}

\end{document}